\documentclass[10pt,twocolumn,letterpaper]{article}

\usepackage{cvpr}      % To produce the REVIEW version
\usepackage{graphicx}
\usepackage{amsmath}
\usepackage{amssymb}
\usepackage{amsthm}
\usepackage{cancel}
\usepackage{booktabs}
\usepackage{adjustbox}
\usepackage[pagebackref,breaklinks,colorlinks]{hyperref}

\usepackage[capitalize]{cleveref}
\crefname{section}{Sec.}{Secs.}
\Crefname{section}{Section}{Sections}
\Crefname{table}{Table}{Tables}
\crefname{table}{Tab.}{Tabs.}

\def\cvprPaperID{11019} % *** Enter the CVPR Paper ID here
\def\confName{CVPR}
\def\confYear{2022}

\usepackage{amsmath,amsfonts,bm}

\newcommand{\ts}[1]{{({#1})}}

\def\eqref#1{equation~\ref{#1}}
\def\1{\bm{1}}

\def\vzero{{\bm{0}}}

\def\vc{{\bm{c}}}

\def\ve{{\bm{e}}}

\def\vh{{\bm{h}}}

\def\vn{{\bm{n}}}

\def\vp{{\bm{p}}}

\def\vr{{\bm{r}}}
\def\vs{{\bm{s}}}
\def\vt{{\bm{t}}}
\def\vu{{\bm{u}}}
\def\vv{{\bm{v}}}

\def\vx{{\bm{x}}}
\def\vy{{\bm{y}}}

\def\mH{{\bm{H}}}
\def\mI{{\bm{I}}}

\def\mO{{\bm{O}}}

\def\mR{{\bm{R}}}

\def\mV{{\bm{V}}}

\def\mX{{\bm{X}}}
\def\mY{{\bm{Y}}}

\DeclareMathAlphabet{\mathsfit}{\encodingdefault}{\sfdefault}{m}{sl}
\SetMathAlphabet{\mathsfit}{bold}{\encodingdefault}{\sfdefault}{bx}{n}

\def\gG{{\mathcal{G}}}

\def\gN{{\mathcal{N}}}

\def\gS{{\mathcal{S}}}
\def\gT{{\mathcal{T}}}

\def\sG{{\mathbb{G}}}

\def\sR{{\mathbb{R}}}

\def\sX{{\mathbb{X}}}
\def\sY{{\mathbb{Y}}}

\newtheorem{propos}{Proposition}

\begin{document}

\title{Equivariant Point Cloud Analysis via Learning Orientations for Message Passing}

\author{Shitong Luo\textsuperscript{*1}, 
Jiahan Li\textsuperscript{*1,2}, 
Jiaqi Guan\textsuperscript{*3}
\phantom{\thanks{~equal contribution.}}
\\
Yufeng Su\textsuperscript{3}, 
Chaoran Cheng\textsuperscript{3}, 
Jian Peng\textsuperscript{3}, 
Jianzhu Ma\textsuperscript{2}\\
\textsuperscript{1} HeliXon Research\\
\textsuperscript{2} Peking University\\
\textsuperscript{3} University of Illinois Urbana-Champaign\\
\texttt{\small luost26@gmail.com,majianzhu@pku.edu.cn} 
}
\maketitle

\begin{abstract}
Equivariance has been a long-standing concern in various fields ranging from computer vision to physical modeling.
Most previous methods struggle with generality, simplicity, and expressiveness --- some are designed ad hoc for specific data types, some are too complex to be accessible, and some sacrifice flexible transformations. 
In this work, we propose a novel and simple framework to achieve equivariance for point cloud analysis based on the message passing (graph neural network) scheme. 
We find the equivariant property could be obtained by introducing an orientation for each point to decouple the relative position for each point from the global pose of the entire point cloud. 
Therefore, we extend current message passing networks with a module that learns orientations for each point.
Before aggregating information from the neighbors of a point, the networks transforms the neighbors' coordinates based on the point's learned orientations.
We provide formal proofs to show the equivariance of the proposed framework.
Empirically, we demonstrate that our proposed method is competitive on both point cloud analysis and physical modeling tasks. 
Code is available at \url{https://github.com/luost26/Equivariant-OrientedMP}.
\end{abstract}

\section{Introduction}
\label{sec:intro}

3D point cloud has become a prevalent data structure for representing a wide range of 3D objects such as 3D scenes \cite{armeni2016s3dis, geiger2013kitti, chang2015shapenet, wu2015modelnet}, molecules \cite{schutt2018schnet, klicpera2019directional, jumper2021alphafold}, and physical particles \cite{fuchs2020se3, thomas2018tensor}.
To capture the geometric relationship among points, message passing networks (graph neural networks) \cite{gilmer2017neural} and their variants such as DGCNN \cite{wang2019dynamic} and PointNet++ \cite{qi2017pointnet++} have become standard tools to model 3D point cloud data.
A characteristic of 3D point cloud is that most of the scalar features of the point cloud are \textit{invariant}\footnote{Note that invariance is a special case of equivariance.} to global rotation and translation, and the vector features of point clouds are \textit{equivariant} to global rotation and translation.
For instance, the category of a shape and its part semantics are not affected by their poses in the 3D space, and the normal vector of each point rotates together with the shape.
Therefore, designing a new type of message passing network that preserves invariance and equivariance for point clouds is a critical challenge with broad impacts on both computer vision and other science domains.

\begin{figure}
    \centering
    \includegraphics[width=1.0\linewidth]{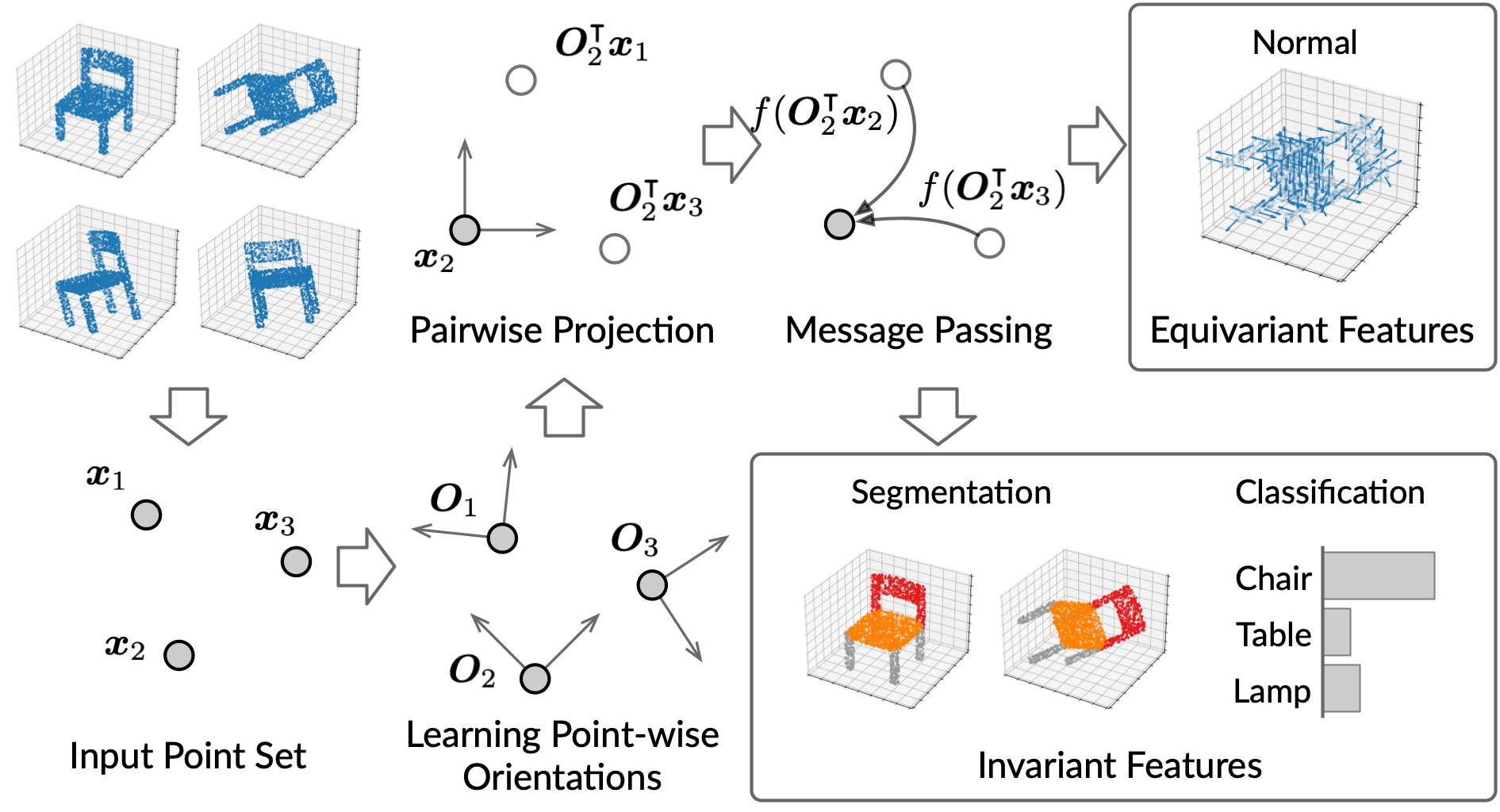}
    \caption{An illustration of the proposed method. It first learns orientations for each point in the point cloud. By projecting the relative coordinate of neighbor points before information aggregation, the model successfully decouples global rotation and achieves equivariance.}
    \label{fig:teaser}
    \vspace{-2.0em}
\end{figure}

The key module of conventional message passing networks for modeling point clouds is to construct the message from point $j$ to $i$ using the difference between 3D coordinates $\vx_j - \vx_i$ which is usually transformed by non-equivariant layers such as regular MLPs.
Hence, it is clear that the outputs of these models are not equivariant to the rotation of point clouds.

Recently, two notable lines of methods have been proposed to address the equivariance problem: tensor field-based neural networks \cite{thomas2018tensor,fuchs2020se3} and vector-based neural networks \cite{deng2021vector,jing2020gvp,schutt2021painn,satorras2021n}.
Tensor field-based networks operate in the 3D space using continuous convolutions and the filters in the model are constrained to be the composite of a learnable radial function and a spherical harmonic.
Though achieving equivariance, they suffer from high space and time complexity because a significant amount of the spherical harmonics need to be calculated on the fly.
The complicated formulation of tensor field networks also makes them less accessible to the community.
Vector-based neural networks are a good alternative to tensor field networks. 
The core part of vector neural networks is their linear layer which takes a list of vectors in 3D space as input and outputs multiple linear combinations of the input vectors. The output vectors are equivariant to input rotations by the nature of the linear combination.
However, as the fully connected layer of vector networks linearly combines input vectors, different dimensions of the vectors are separated, which prohibits flexible vector transformations and hinders information flows between dimensions. In addition, vector-based neural networks are inadequate to model essential geometric relationships such as angles due to the linearity of the vector propagation.

To address this challenge, we introduce a novel equivariant message passing model by learning the orientation of each point during the message passing process. The overall learning process includes three key steps.
First, it learns the orientation matrix for each point.
Second, we project the relative position vector between these two points to the learned orientations, which decouples the relative position of two points from global rotations.
Next, the projected relative position difference is integrated into neural messages by some specific network architecture.
In the end, we can get both invariant and equivariant properties based on the output features of message passing layers along with the learned orientations.

In comparison to previous equivariant models, our method is much simpler but more flexible. Specifically, we can easily extend our model by applying arbitrary transforms to the projected relative position vector to compose the message, while vector-based networks only allow the linear combination of vectors.

To summarize, our main contributions include: 
(1) We propose a new framework for equivariant point cloud analysis and provide formal formulation and proof of the equivariance.
(2) We show the framework's generality by using it to augment classical point cloud networks including DGCNN and RS-CNN.
(3) We conduct extensive experiments on various tasks including point cloud classification, segmentation, normal estimation, and many-body modeling, and demonstrate that the proposed method achieves competitive performance.

\section{Related Work}
\label{sec:related}

\paragraph{Message Passing Networks for Point Clouds}
Message passing neural networks (also known as graph neural networks) \cite{gilmer2017neural} have become a standard architecture to approach machine learning tasks on point clouds.
Representative models include PointNet++ \cite{qi2017pointnet++}, DGCNN \cite{wang2019dynamic}, RS-CNN \cite{liu2019relation}, PointCNN \cite{wu2019pointconv,li2018pointcnn}, and so on \cite{DBLP:conf/cvpr/DengBI18}.
For each point in the point cloud, a message passing layer is first applied to aggregate information from features and relative positions of its neighbor points. The aggregated information is transformed to point representation by using standard neural networks such as MLPs. By definition, directly using MLPs to process coordinates is not equivariant to rotation and one can observe that a rotation of the 3D objects might yield different predictions for these models. 
Recently, there has been a growing interest in designing new deep learning architectures which can preserve equivariant and invariant properties.

\paragraph{Equivariance via Orientation Estimation}
Estimating canonical orientations for objects is a way to decouple the global rotation and achieve equivariance. 
PointNet \cite{qi2017pointnet} constructs a sub-network based on standard MLPs to estimate a canonical orientation in the form of a 3x3 matrix for each input shape. However, the matrix is not restricted to be an orthogonal rotation matrix and it is not equivariant to the input either. PointNet can only achieve approximate equivariance by data augmentation.
GC-Conv \cite{zhang2020gcconv} uses principal component analysis (PCA) to define canonical orientations at different scales but its orientations rely on handcrafted prior knowledge which is not available for many real-world applications. 
LGR-Net \cite{zhao2019lgrnet} relies on normal vectors as additional inputs, which could be viewed as a special form of orientation for each point. Yet, in most settings, normal vectors are hard to acquire or even not defined.
Quaternion equivariant capsule network \cite{zhao2020quaternion} uses quaternions to represents point-wise orientations. However, it requires initial orientations as input, while our methods learn orientations in a fully end-to-end manner.
Other works \cite{brachmann2014learning6d, li2020category, wang2019normalized, rhodin2018unsupervised,sun2020canonical} classified to this class mainly focus on pose estimation and most of them require the supervision by the ground-truth canonical pose and are only equivariant by data augmentation.

\paragraph{Equivariance via Vector-Based Networks}
Recently, vector-based neural networks have emerged as a new approach to achieve equivariance \cite{jing2020gvp, deng2021vector, satorras2021n, schutt2021painn}.
They generalize scalar activations in standard neural networks to vector activations.
To process these vector activations, vector-based linear layer is designed to take a list of vectors as input and output multiple linear combinations of the input vectors, which makes the entire network equivariant by nature.
For the activation functions, GVP \cite{jing2020gvp} scale the vectors based on their norm. Consequently, vectors passed through such activation layers remain a linear combination of the input vectors. VNN \cite{deng2021vector} introduces a non-linear layer that outputs two sets of vectors and projects one to another. Since the two sets of vectors are both equivariant to inputs, the projected vectors are also equivariant.
The major limitation of vector-based networks arises from linear combinations --- as their fully connected layer linearly combines input vectors. This prohibits flexible vector transformations and hinders information flows among vector dimensions.

\paragraph{Other Equivariant Networks}
Tensor field-based networks \cite{thomas2018tensor,fuchs2020se3,poulenard2021functional} are another thread of works which could achieve equivariance. They rely on constraining the convolutional kernel to the spherical harmonics family. By restricting the space of learnable functions, features convolved with these kernels are provably equivariant to rotations and translations. However, it takes much space to cache the spherical harmonics calculated for irregular point clouds, leading to high space and time complexity.
Concurrently, many works take efforts to design invariant operators \cite{schutt2018schnet,weiler20183d,klicpera2020directional, qiao2021unite,batzner2021se, joseph2019momen, li2021closer}, based on rotationally invariant measures such as distances and angles.

\section{Preliminaries}
\label{sec:background}

\subsection{Notations}

We denote a point cloud as $\mX = ( \vx_1, \ldots, \vx_N )$, where each point is represented as a 3D vector $\vx_i \in \sR^3$.
The orientation of point $i$ is defined as a 3x3 rotation matrices $\mO_i$. The optimized orientation matrices should satisfy orthogonality constraint that $\mO_i^\intercal = \mO_i^{-1}$ and has unit determinant ($\det \mO_i = 1$)\footnote{We can always choose the right-handed orientation so that the determinant is 1 instead of -1.}.

\subsection{Equivariance}
\label{sec:prelim:equiv}

Let $\sG$ denote a transformation group (\eg rotation group, permutation group, \etc), $g \in \sG$ denote a specific transform in the group, and $\gT_g(x): \sX \rightarrow \sX$ denote the function that applies transform $g$ to the the input $x \in \sX$.
A function $f: \sX \rightarrow \sY$ is equivariant to $\sG$ if there exists a transform function on the output domain $\gS_g(y): \sY \rightarrow \sY$ such that for all input $x\in\sX$, the following equation holds:
\begin{equation}
    f\left(\gT_g \left( x \right) \right) = \gS_g\left( f\left( x \right) \right).
\end{equation}
In point cloud analysis, we consider the following types of equivariance:
\begin{itemize}
    \setlength\itemsep{0em}
    \item \textbf{Rotational and translational invariance}: The shape category of point clouds and point-wise semantic labels fall into this type of equivariance. 
    Specifically, $\sG$ is the group of rigid transform (rotation and translation) whose elements can be represented by a rotation matrix and a translation vector: $(\mR, \vt)$.
    The transform function on the input domain is defined as $\gT_{(\mR, \vt)}(\mX) = (\mR \vx_1 + \vt, \ldots, \mR \vx_N + \vt)$. The transform function on the output domain (the domain of shape categories and semantic labels) is simply an identity transform: $\gS_{(\mR, \vt)}(\vc) = \vc$. In other words, a point cloud $\mX$'s category or semantic labels $\vc$ keeps unchanged regardless of the rotations and translations applied to $\mX$.
    
    \item \textbf{Rotational equivariance and translational invariance}: Point-wise normal vectors and velocities of physical particles fall into this class.
    In particular, $\sG$ is also the group of rigid transform. $\gT_{(\mR, \vt)}$ is defined as $\gT_{(\mR, \vt)}(\mX) = (\mR \vx_1 + \vt, \ldots, \mR \vx_N + \vt)$, and $\gS_{(\mR, \vt)}$ is defined as
    $\gS_{(\mR, \vt)}\left( \vv_1, \ldots, \vv_N \right) = \left( \mR \vv_1, \ldots, \mR \vv_N \right)$. Here, $\vv_i \in \sR^3$ can represent the point-wise normal vector or velocity in the N-Body system modeling problem.
    
    \item \textbf{Permutation equivariance and invariance}: 
    These types of equivariance are not among the interest of this work because most message passing networks are already permutation equivariant for point-wise features and invariant for global features. They can be trivially combined with the geometric equivariance discussed above. 
    We state them formally here for clarity: $\gG$ is the permutation group, and a permutation is denoted as a bijective mapping $\sigma$ on $\{1, \ldots, N\}$.
    The transform function is $\gT_\sigma(\mX) = (\vx_{\sigma(1)}, \ldots, \vx_{\sigma(N)})$.
    For point-wise features, the outputs are equivariant to permutation: $\gS^\text{eqv}_\sigma(\mY) = (\vy_{\sigma(1)}, \ldots, \vy_{\sigma(N)})$. For global features, the outputs are invariant: $\gS^\text{inv}_\sigma(\vy) = \vy$.
\end{itemize}

Based on these definitions, it is clear that invariance is a special case of equivariance. In this paper, we use both of the terms ``equivariance'' and ``invariance'' to distinguish cases, but sometimes we also use ``equivariance'' as a general concept that implies both equivariance and invariance.

\subsection{Intuition: Inherent Orientations of Points}
\label{sec:prelim:intuition}
One of the \textit{key intuitions} behind our method is that: each point in a point cloud has an inherent orientation when considering its context, although an individual point itself is fully symmetric.
As shown in \ref{fig:orientation}, the seat of the chair has at least two inherent orientations: ``UP" and ``FRONT" --- ``UP" is the orientation to which the seat faces and ``FRONT'' is the direction opposite to the back of the chair.
If the chair is represented as a point cloud, then points on the seat inherit these orientations, and we can denote the orientation using a rotation matrix $\mO_i$ where we assume the first and second column vectors represent ``UP'' and ``FRONT''. These rotation matrices are naturally equivariant to the global rotation of the chair.
Assuming point $i$ is on the seat and point $j$ is any other point in the point cloud, then $\mO_i^\intercal(\vx_j - \vx_i)$ essentially converts the difference of two points' global coordinates to the relative position of $j$ to $i$ in terms of the inherent orientation of $i$.
The first component of $\mO_i^\intercal(\vx_j - \vx_i)$ indicates the height of point $j$ relative to the seat and the second component represents how far is the point in front of the chair.
Therefore, $\mO_i^\intercal(\vx_j - \vx_i)$ is invariant to the pose of the chair and can go through arbitrary transforms without breaking invariance.
This concept is illustrated in Figure~\ref{fig:inherent}. Notice that no matter how the chair rotates, the relative coordinate of the yellow circle with respect to the FRONT-UP orientation remains unchanged. However, in most real-world applications, we have no prior knowledge about the orientation for each point. Therefore, it is necessary to design a network that can learn orientations for each point in an end-to-end fashion.

\begin{figure}
    \centering
    \includegraphics[width=0.9\linewidth]{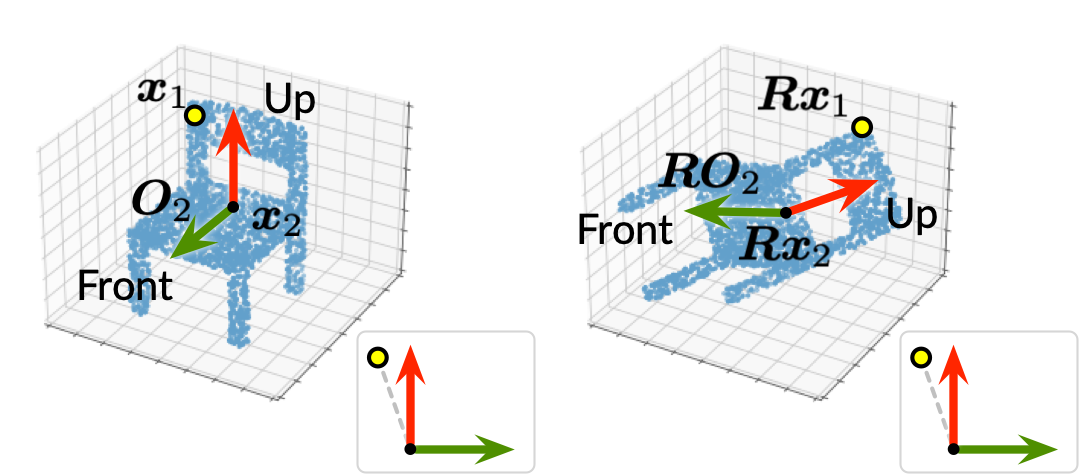}
    \caption{An illustration of the concept ``inherent orientation''. The coordinate of the yellow dot is $\vx_1$, and the coordinate of the black dot is $\vx_2$. The arrows are two basis vectors of the black dot's inherent orientation denoted as $\mO_2$. The yellow dot's position relative to $\vx_2$ and $\mO_2$ keeps unchanged despite the global rotation $\mR$.}
    \label{fig:inherent}
    \vspace{-1.2em}
\end{figure}
\section{Method}
\label{sec:method}

The central idea of this work is to learn orientations for message passing. 
This section elaborates on each building block, organized as follows:
\begin{itemize}
    \setlength\itemsep{0em}
    \item In Section~\ref{sec:method:orient}, we detail the principle of orientation learning and propose a neural network architecture to learn orientations.
    \item In Section~\ref{sec:method:msgpass}, we show how the designed architecture could be applied to the learned orientations to message passing networks and how it preserves equivariance.
    \item In Section~\ref{sec:method:output}, we formulate the readout layers for estimating both equivariant and invariant properties.
\end{itemize}

\subsection{Learning Orientations}
\label{sec:method:orient}

As introduced in Section~\ref{sec:prelim:intuition}, the orientation of a point is inherent and contextual. Point-wise orientations can be represented as orthogonal orientation matrices $\mO_i \in \sR^{3\times3}$ ($i = 1, \ldots, N$).
Therefore, in order to learn effective point-wise orientations, there should be a sub-network that (1) captures the context of each point, (2) outputs orthogonal orientation matrices, and (3) is equivariant to the global pose.
To fulfill objective (1) and (3), we employ a variant of GVP-GCN \cite{jing2020gvp, satorras2021n} (a vector-based graph neural network) over the $k$-nearest-neighbor graph of the point cloud to estimate two equivariant vector features for each point.
Subsequently, we leverage on the Gram-Schmidt process to construct orientation matrices from the vectors \cite{zhou2019continuity}, which fulfills objective (2).
Below is the description of the network for learning orientations:

Let $\mX = (\vx_1, \ldots, \vx_N)$ denote the input point cloud.
The point cloud, with all zeros as its initial scalar and vector features, is passed to the first vector-based graph convolution layer denoted as $\text{V-GConv}$:
\begin{align}
\label{eq:vgconv-layer1}
    (\mH^\ts{1}, \mV^\ts{1}) \gets \text{V-GConv}_{1}(\mX, \vzero, \vzero).
\end{align}
Here, $\mH^\ts{1} = (\vh^\ts{1}_1, \ldots, \vh^\ts{1}_N)$ is the scalar features of the point cloud, and $\mV^\ts{1} = \left[ (\vv^\ts{1}_{11}, \vv^\ts{1}_{12}, \ldots), \ldots, (\vv^\ts{1}_{N1}, \vv^\ts{1}_{N2}, \ldots) \right]$ denote its vector features.
Next, the scalar and vector features are iteratively updated through a stack of $\text{V-GConv}$ layers:
\begin{align}
    (\mH^\ts{\ell}, \mV^\ts{\ell}) \gets \text{V-GConv}_{\ell}(\mX, \mH^\ts{\ell-1}, \mV^\ts{\ell-1}).
\end{align}
The final layer outputs two vectors for each point denoted as $\mV^\ts{L} = \left[ (\vv^\ts{L}_{11}, \vv^\ts{L}_{12}), \ldots, (\vv^\ts{L}_{N1}, \vv^\ts{L}_{N2}) \right]$.
Details about $\text{V-GConv}$ layers are presented in the supplementary material.

We then employ Gram-Schmidt process to normalize and orthogonalize the two vectors for each point:
\begin{align}
\label{eq:ortho1}
    \vu_{i1} & \gets \frac{\vv^\ts{L}_{i1}}{\| \vv^\ts{L}_{i1} \|}, \\
\label{eq:ortho2}
    \vu^\prime_{i2} & \gets \vv^\ts{L}_{i2} - \langle\vv^\ts{L}_{i2},\vu_{i1}\rangle \vu_{i1}, \\
\label{eq:ortho3}
    \vu_{i2} & \gets \frac{\vu^\prime_{i2}}{\| \vu^\prime_{i2} \|},
\end{align}
Here $\langle\cdot,\cdot\rangle$ is the inner product of two vectors. 
Finally, we construct point-wise orientation matrices by combining the two orthogonal unit vectors $(\vu_{i1}, \vu_{i2})$ and their cross product:
\begin{equation}
\label{eq:ort_matrix}
    \mO_i \gets \left[ \vu_{i1}, \vu_{i2}, \vu_{i1} \times \vu_{i2} \right] \quad (i = 1, \ldots, N).
\end{equation}

Orientation matrices constructed in this way fulfill the three requirements summarized at the beginning of this section.
In particular, the $\text{V-GConv}$ layers enable the learned orientations to sense their contexts and guarantee equivariance. The Gram-Schmidt orthogonalization and Eq.\ref{eq:ort_matrix} ensure the orthogonality of learned orientations.
The equivariance of learned orientations is stated formally in the following proposition:

\begin{propos}
\label{prop:orientation}
The proposed network for learning orientations, denoted as $f_\text{ort}(\vx_1, \ldots, \vx_N) = (\mO_1,\ldots,\mO_N)$, is equivariant to rotation and invariant to translation, satisfying $f_\text{ort}\left( \mR\vx_1 + \vt, \ldots, \mR\vx_N + \vt \right) = \left( \mR\mO_1, \ldots, \mR\mO_N \right)$. 
\end{propos}
\begin{proof}
See the supplementary material.
\end{proof}

\noindent\textbf{Extension: Learning Global Orientations}
The network for learning point-wise orientations can be adapted for learning global orientations.
We first construct a hierarchy of the point cloud: $\mX^\ts{L} = \{ \bar{\vx} \} \subset \mX^\ts{L-1} \subseteq \ldots \subseteq \mX^\ts{1} \subseteq \mX$, where $\bar{\vx} = \frac{1}{N}\sum \vx_i$.
The $\ell$-th $\text{V-GConv}$ layer ($\ell < L$) finds the $k$-nearest-neighbors of a point $\vx_i^\ts{\ell}$ in $\mX^\ts{\ell-1}$ to aggregate information. The scalar and vector features of $\bar{\vx}$ is obtained via global pooling over $\mX^\ts{L-1}$'s features.
Finally, the orientation constructed from $\bar{\vx}$'s vectors feature serves as the global orientation $\mO_\text{glob}$.
As the geometric center $\bar{\vx}$ is equivariant, it is straightforward to see that $\mO_\text{glob}$ is invariant to translation and equivariant to rotation.

\begin{figure}
    \centering
    \includegraphics[width=0.95\linewidth]{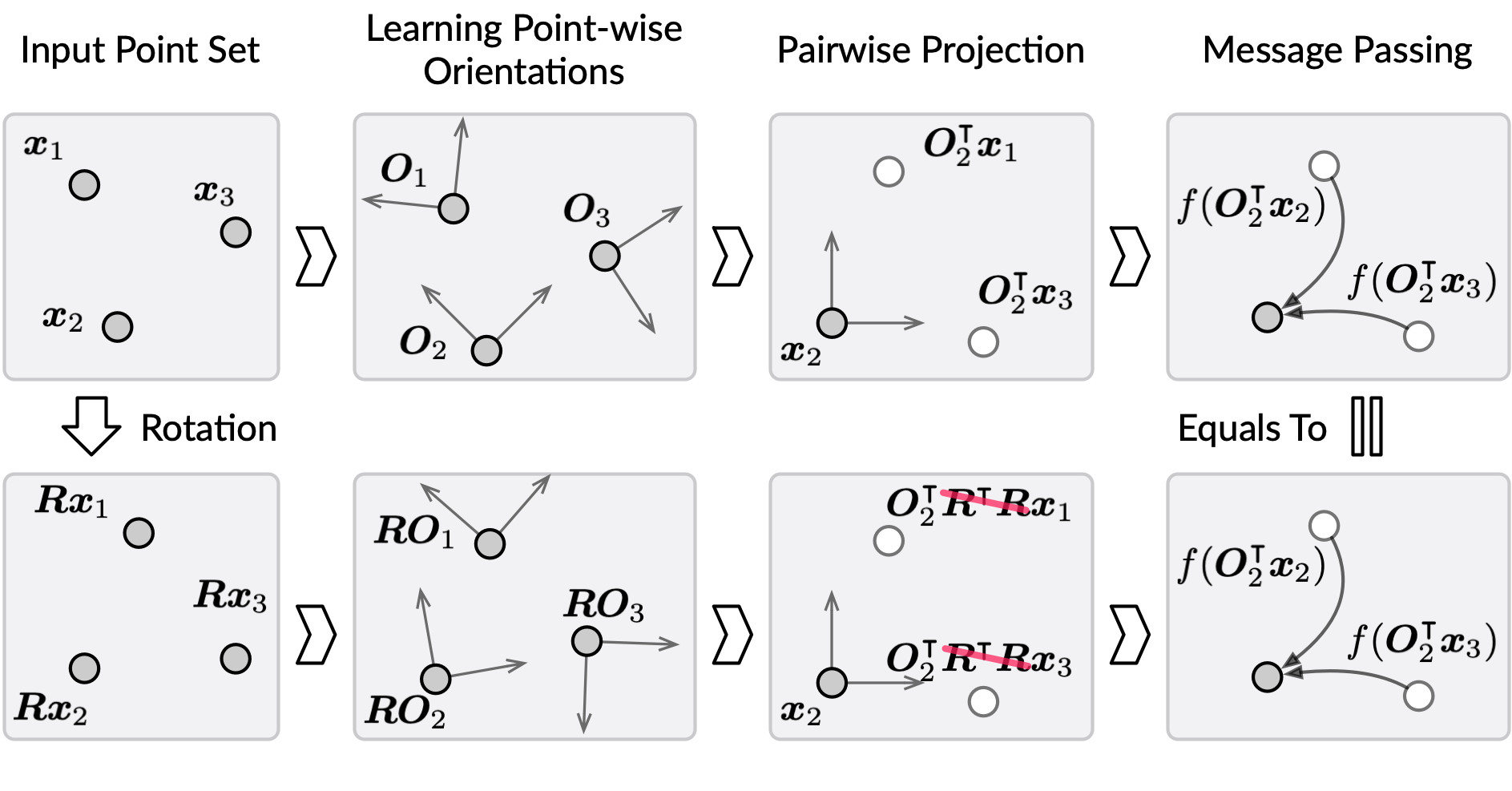}
    \caption{An illustration of message passing with orientations and its equivariant nature. Note that a proper rotation matrix $\mR$ should satisfy the constraint $\mR^\intercal\mR = \mI$.}
    \label{fig:invariant-mp}
    \vspace{-1.2em}
\end{figure}

\subsection{Equivariant Message Passing with Learned Orientations}
\label{sec:method:msgpass}
Typical message passing layers for point clouds use coordinate differences $\vx_j - \vx_i$ to compose the message passed from one point to another. They can be formulated in the following general form:
\begin{equation}
    \vh^\ts{\ell+1}_i \gets \underset{j \in \gN(i)}{\operatorname{Agg}}  H\left( \vh^\ts{\ell}_i, \vh^\ts{\ell}_j, \vx_j - \vx_i \right),
\end{equation}
where $\operatorname{Agg}$ is a permutation-invariant aggregation operator (common choices include max, sum, and mean). $H(\cdot)$ is an arbitrary neural network such as an MLP.
To achieve invariance, we project $\vx_j - \vx_i$ using the learned orientation $\mO_i$ and rewrite the formula as:
\begin{equation}
\label{eq:oriented_general}
    \vh^\ts{\ell+1}_i \gets \underset{j \in \gN(i)}{\operatorname{Agg}} H\left( \vh^\ts{\ell}_i, \vh^\ts{\ell}_j, \mO^\intercal_i(\vx_j - \vx_i) \right).
\end{equation}

By applying $\mO^\intercal_i$ to $\vx_j - \vx_i$, we decouple the coordinate differences to the pose of the point cloud --- the rotated relative position $\mO^\intercal_i(\vx_j - \vx_i)$ is invariant to global rotations of the point cloud. 
Figure \ref{fig:invariant-mp} illustrates this message passing scheme.
Formally, we have:
\begin{propos}
\label{prop:invariant-feature}
Under the assumption that $(\mO_1, \ldots, \mO_N)$ comes from an equivariant function $g$ satisfying $g(\mR\mX + \vt) = (\mR\mO_1, \ldots, \mR\mO_N)$, and that $\mH$ comes from an invariant function $h$ satisfying $h(\mR\mX + \vt) = \mH$,
the message passing formula in Eq.\ref{eq:oriented_general}, denoted as $f^\ts{\ell}_\text{mp}$, is invariant to rotation and translation and satisfies $f^\ts{\ell}_\text{mp}\left(\mR\mX+\vt \right) = f^\ts{\ell}_\text{mp}\left(\mX \right)$.
\end{propos}
\begin{proof}
See the supplementary material.
\end{proof}

Notice that the output feature $\vh^\ts{\ell+1}_i$ is invariant to rotations despite the form of $H$. This allows arbitrary flexible network design for $H$, while previous work imposes strong prior assumptions on $H$ such as the linear combination of vectors, handcrafted invariant measures, or spherical harmonics filtering. 
In addition, projecting $\vx_j - \vx_i$ using $\mO^\intercal_i$ involves the inner product between $\vx_j - \vx_i$ and $\mO^\intercal_i$'s column vectors which explicitly encodes angular geometric information.

\paragraph{Augmented Message Passing Networks}
In this section, we demonstrate how to endow classical point-based message passing networks with equivariance by using RS-CNN and DGCNN as examples.

\textbf{RS-CNN}.
The message passing formula of vanilla RS-CNN is:
\begin{equation}
    \vr_{ij}\gets \vx_j - \vx_i,
\end{equation}
\begin{equation}
    \vh^\ts{\ell+1}_i \gets \underset{j \in \gN(i)}{\max} \left( M\left(\| \vr_{ij} \|, \vr_{ij}\right) \odot \vh^\ts{\ell}_j \right),
\end{equation}
where $\max$ denotes the max pooling operation, $M$ is an MLP network, and $\odot$ denotes element-wise multiplication. We can rewrite RS-CNN by incorporating Eq.\ref{eq:ort_matrix} as the follow:
\begin{equation}
    \vh^\ts{\ell+1}_i \gets \underset{j \in \gN(i)}{\max} \left( M\left(\| \mO^\intercal_i\vr_{ij} \|, \mO^\intercal_i\vr_{ij}\right) \odot \vh^\ts{\ell}_j \right),
\end{equation}
In vanilla RS-CNN, the initial point-wise features $\vh^\ts{0}_i$ are set to the points' 3D coordinates, which would break equivariance in our setting.
As a remedy, we can either initialize $\vh^\ts{0}_i$ with a constant or coordinates aligned to the global orientation $\mO_\text{glob}$. 

\textbf{DGCNN}.
The formulations of the first and subsequent layers of vanilla DGCNN are:
\begin{align}
\label{eq:vanilla-dgcnn-0}
    \vh^\ts{1}_i & \gets \underset{j \in \gN(\vx_i)}{\max} \sigma \left( M_0(\vx_j - \vx_i, \vx_i) \right), \\
\label{eq:vanilla-dgcnn-l}
    \vh^\ts{\ell+1}_i & \gets \underset{j \in \gN(\vh^\ts{\ell}_i)}{\max} \sigma \left( M_\ell(\vh^\ts{\ell}_j - \vh^\ts{\ell}_i, \vh^\ts{\ell}_i) \right),
\end{align}
where $\sigma$ is the activation function and $M_\ell$ is a MLP layer.
To obtain the equivariant property, we modify Eq.\ref{eq:vanilla-dgcnn-0} to:
\begin{equation}
\label{eq:ort-dgcnn-0}
    \vh^\ts{1}_i \gets \underset{j \in \gN(\vx_i)}{\max} \sigma \left( M_0(\mO^\intercal_i(\vx_j - \vx_i)) \right).
\end{equation}
Here, we ignore $\vx_i$ from the message passing formula, but we can still apply $\mO_\text{glob}^\intercal$ to $\vx_i$ to decouple the point's coordinate from global rotation.

The above formulations are kept simple for clarity, we introduce the following extensions that may further increase the power of the network:

\noindent\textbf{Extension 1: Multiple Orientations} 
The relative coordinates can be projected using multiple orientation matrices. Differently projected relative coordinates can be concatenated to feed to the message network $H(\cdot)$.

\noindent\textbf{Extension 2: Incorporating Global Information}
In some cases, global spatial information can be useful. We can project global coordinates using $\mO_\text{glob}$ to obtain global invariant canonical coordinates and use them as point-wise invariant features.

\subsection{Estimating Equivariant and Invariant Properties}
\label{sec:method:output}
So far, we have introduced how to extract equivariant point-wise features in previous sections.
The last step is to predict properties based on the learned features.
We divide properties into two classes: invariant properties and equivariant properties.
For example, invariant properties include the category of a point cloud, point-wise semantics, and so on. 
They are independent of the global pose of the point cloud and therefore belong to the first case in Section~\ref{sec:prelim:equiv}.
Equivariant properties such as normal vectors and velocities rotate together with the global rotation, and therefore they belong to the second case discussed in Section~\ref{sec:prelim:equiv}.

\paragraph{Invariant Property Estimation}
Point-wise features output by the proposed message passing layers are already invariant according to Propos.\ref{prop:invariant-feature}.
Therefore, to predict point-wise invariant property, we can directly feed the features into an arbitrary neural network (commonly an MLP) to estimate the property.
For global properties, we can perform pooling operations over all the features. The pooled global features are clearly invariant and we can safely feed it to a neural network to estimate global invariant properties.

\paragraph{Equivariant Property Estimation}
The learned point-wise features are invariant to rotation. Hence, in order to predict equivariant properties, we leverage on the learned orientations again.
Specifically, we first use an MLP to transform point-wise features to 3D vectors. 
Then, we apply the orientations to the vectors to obtain the equivariant vector properties.
These two steps can be formulated as:
\begin{align}
\label{eq:equiv-property-1}
    \vp_i & \gets \operatorname{MLP}\left( \vh^\ts{L}_i \right),\\
\label{eq:equiv-property-2}
    \ve_i & \gets \mO_i \vp_i,
\end{align}
where $\ve_i$ is the desired equivariant for point $i$.
The equivariance of $\ve_i$ is stated formally in the following proposition:
\begin{propos}
\label{prop:equiv-property}
Under the assumption that $(\mO_1, \ldots, \mO_N)$ comes from an equivariant function $g$ satisfying $g(\mR\mX + \vt) = (\mR\mO_1, \ldots, \mR\mO_N)$, and that $\mH^\ts{L}$ comes from an invariant function $h$ satisfying $h(\mR\mX + \vt) = \mH^\ts{L}$,
The property estimator $f_\text{prop}$ defined by Eq.\ref{eq:equiv-property-1} and Eq.\ref{eq:equiv-property-2} is equivariant to rotation and invariant to translation, 
satisfying $f_\text{prop}\left(\mR\mX+\vt \right) = \mR f_\text{prop}\left(\mX \right)$.
\end{propos}
\begin{proof}
See the supplementary material.
\end{proof}
By combining proposition \ref{prop:orientation}, \ref{prop:invariant-feature}, and \ref{prop:equiv-property}, we can easily see that message passing networks in the proposed framework are equivariant.

\section{Experiments}
\label{sec:experiments}

We evaluate our method on three types of tasks including point cloud analysis (Section~\ref{sec:expr:pointcloud}), physical modeling (Section~\ref{sec:expr:nbody}), and analytical study (Section~\ref{sec:expr:analytical}). In the point cloud analysis, the shape classification task and part segmentation task examine the method's capability of predicting global and point-wise invariant properties. The normal estimation task evaluates the model's ability on estimating equivariant properties. The task of physical modeling is to predict particle velocities in an n-body system, which heavily relies on the model's equivariance to achieve accurate predictions. In the end, we provide more insight into the behavior of our model by visualizing learned orientations of different 3D objects.

\subsection{Point Cloud Analysis}
\label{sec:expr:pointcloud}

\begin{table*}
    \centering
    \caption{Point cloud segmentation results in the IoU (\textit{Inter-over-Union}). In z/SO(3) setting, our model outperforms other baselines on 11 out of 16 shapes.} 
    % \vspace{-0.8em}
    \label{tab:partseg}
    \begin{adjustbox}{width=1.0\textwidth}
    \begin{tabular}{l|cccc cccc cccc cccc c}
\toprule
z/SO(3) & plane & bag & cap & car & chair & earph. & guitar & knife & lamp & laptop & motor & mug & pistol & rocket & skate & table \\ 
\hline

PointNet \cite{qi2017pointnet} & 40.4 & 48.1 &  46.3 &  24.5 &  45.1 &  39.4 &  29.2 &  42.6 &  52.7 &  36.7 &  21.2 &  55.0 &  29.7 &  26.6 &  32.1 &  35.8 \\

PointNet++ \cite{qi2017pointnet++}& 51.3 & 66.0 &  50.8 &  25.2 &  66.7 &  27.7 &  29.7 &  65.6 &  59.7 &  70.1 &  17.2 &  67.3 &  49.9 &  23.4 &  43.8 &  57.6 \\

PointCNN \cite{li2018pointcnn} & 21.8 & 52.0 &  52.1 &  23.6 &  29.4 &  18.2 &  40.7 &  36.9 &  51.1 &  33.1 &  18.9 &  48.0 &  23.0 &  27.7 &  38.6 &  39.9 \\

DGCNN \cite{wang2019dynamic} & 37.0 & 50.2 &  38.5 &  24.1 &  43.9 &  32.3 &  23.7 &  48.6 &  54.8 &  28.7 &  17.8 &  74.4 &  25.2 &  24.1 &  43.1 &  32.3 \\

ShellNet \cite{zhang2019shellnet} & 55.8 & 59.4 &  49.6 &  26.5 &  40.3 &  51.2 &  53.8 &  52.8 &  59.2 &  41.8 &  28.9 &  71.4 &  37.9 &  49.1 &  40.9 &  37.3 \\ 
\hline

RI-Conv \cite{zhang2019rotation} & 80.6 & 80.0 &  70.8 &  68.8 &  86.8 &  70.3 &  87.3 &  84.7 &  77.8 &  80.6 &  57.4 &  91.2 &  71.5 &  52.3 &  66.5 &  78.4 \\

GC-Conv \cite{zhang2020global} & 
    80.9 & \bf 82.6 & 81.0 & 70.2 & 88.4 & 70.6 & 87.1 & \bf 87.2 & 
    \bf 81.8 & 78.9 & 58.7 & 91.0 & 77.9 & 52.3 & 66.8 & 80.3 \\

RI-Fwk. \cite{li2021rotation} & {81.4} & {82.3} &  \bf 86.3 &  {75.3} &  {88.5} &  {72.8} &  {90.3} &  82.1 &  {81.3} &  {81.9} &  {67.5} &  {92.6} &  {75.5} & {54.8}  & {75.1} &  {78.9} \\

\hline

Ours & 
    \bf 81.7 & 79.0 & 85.0 & \bf 78.1 & \bf 89.7 & \bf 76.5 & \bf 91.6 & 85.9 &
    81.6 & \bf 82.1 & \bf 67.6 & \bf 95.0 & \bf 79.6 & \bf 64.4 & \bf 76.9 & \bf 80.7
    \\
\bottomrule

\end{tabular}
    \end{adjustbox}
    % \vspace{-1.0em}
\end{table*}

\begin{table}
    \centering
    \caption{Consine distance between the predicted and ground truth normals. Our model outperforms all the baselines by a clear margin. In addition, compared to non-equivariant baselines, our model shows consistent performance thanks to equivariance.} 
    % \vspace{-0.8em}
    \label{tab:normal}
    \begin{adjustbox}{width=0.85\linewidth}
    \begin{tabular}{l|ccc}
\toprule
Methods & $z/z$ & $z/\text{SO}(3)$ & $\text{SO}(3)/\text{SO}(3)$  \\
\midrule
PointNet++ \cite{qi2017pointnet++} & 0.34 & 0.81 & 0.55 \\
RS-CNN \cite{liu2019relation} & 0.26 & 0.83 & 0.50 \\
DGCNN \cite{wang2019dynamic} & 0.29 & 0.32 & 0.22 \\ 
\midrule
RI-Conv \cite{zhang2019rotation} & 1.33 & 1.30 & 1.30 \\
GC-Conv \cite{zhang2020global} & 0.42 & 0.44 & 0.42 \\
\midrule
Ours & \bf 0.20 & \bf 0.20 & \bf 0.20 \\

\bottomrule
\end{tabular}

    \end{adjustbox}
    % \vspace{-1.0em}
\end{table}

% \vspace{-0.3em}
\subsubsection{Classification (Invariant)}
% \vspace{-0.5em}

A point cloud's category is its \textit{invariant} property. 
This task can demonstrate the model's capability of learning invariant representations.

\vspace{-1em}

\paragraph{Setup}
We use the ModelNet40 \cite{wu2015modelnet} dataset and its official train-test split for the point cloud classification task.
The dataset consists of 12,311 shapes from 40 different categories, of which 2,468 shapes are left out for test. 
Each point cloud contains 1,024 points.

Following the setup in the previous work \cite{deng2021vector,li2021rotation,zhang2020global}, we adopt three different train-test rotation settings: 
(1) $z/z$: both training and test point clouds are rotated around the gravitational axis; 
(2) $z/\text{SO}(3)$: training point clouds are rotated around the gravitational axis and test point clouds are rotated arbitrarily. This setting examines the model's quality of equivariance-by-construction.
(3) $\text{SO}(3)/\text{SO}(3)$: both training and test point clouds are rotated arbitrarily. 

\vspace{-1em}

\paragraph{Results}
\begin{table}
    \centering
    \caption{Classification accuracy on ModelNet40. The upper rows are non-equivariant models, and the lower rows are equivariant models. The performance of our model is on par with equivariant baselines.} 
    \vspace{-0.8em}
    \label{tab:cls}
    \begin{adjustbox}{width=0.9\linewidth}
    \begin{tabular}{l|ccc}
\toprule
Methods & $z/z$ & $z/\text{SO}(3)$ & $\text{SO}(3)/\text{SO}(3)$  \\
\midrule
PointNet \cite{qi2017pointnet}  & 85.9 & 19.6 & 74.7 \\
PointNet++ \cite{qi2017pointnet++}  & 91.8 & 28.4 & 85.0 \\
RS-CNN \cite{liu2019relation} & 90.3 & 48.7 & 82.6 \\
DGCNN \cite{wang2019dynamic}    & 90.3 & 33.8 & 88.6 \\

\midrule

RI-Conv \cite{zhang2019rotation}    & 86.5 & 86.4 & 86.4 \\
GC-Conv \cite{zhang2020global}  & 89.0 & 89.1 & 89.2 \\
\midrule
Ours-RSCNN & 87.6 & 87.6 & 87.5 \\
Ours-DGCNN & 88.4 & 88.4 & 88.9 \\

\bottomrule
\end{tabular}

    \end{adjustbox}
    \vspace{-1.0em}
\end{table}

We compare our model with both equivariant and non-equivariant baselines.
Table \ref{tab:cls} shows the test classification accuracy of different methods. 
Our model achieves competitive performance compared to other baseline models.

% \vspace{-0.5em}
\subsubsection{Part Segmentation (Invariant)}
\label{sec:expr:pointcloud:partseg}
% \vspace{-0.5em}

Point-wise labels are also an \textit{invariant} property of a point cloud. This task can examine the model's performance on modeling invariant detail properties.

\vspace{-1em}

\paragraph{Setup}
We evaluate the models on the ShapeNet \cite{yi2016scalable} dataset which consists of 16 categories with 16,881 shapes, and are labeled in 50 parts in total. Following the convention, we sample 2,048 points for each shape.
We use two train-set rotation settings for this task: z/SO(3) and SO(3)/SO(3).
Our model employs the equivariant adaptation DGCNN (see Section \ref{sec:method:msgpass}) as the backbone.
We calculate the IoU (Inter-over-Union) metric to measure the segmentation quality for each category.

\vspace{-1em}

\paragraph{Results}
Table \ref{tab:partseg} summarizes the quantitative comparisons with baseline methods in z/SO(3) setting. The results of SO(3)/SO(3) can be found in the supplementary material.
Our model outperforms all the baseline models on 12 out of 16 categories in z/SO(3) setting and 10 of 16 categories in SO(3)/SO(3) setting.
This improvement demonstrates that our equivariant model is effective in learning fine-grained details.

% \vspace{-0.5em}
\subsubsection{Normal Estimation (Equivariant)}
% \vspace{-0.5em}

Normal vectors are \textit{equivariant} property of a point cloud as they rotate together with the global rotation of the point cloud.
This task tests the models' capabilities of modeling equivariant properties.

\vspace{-1em}
\paragraph{Setup}
We validate our model and other baselines using ModelNet40\cite{wu2015modelnet}, where each point in the point cloud is labeled with a 3D normal vector.
We use the same data split and rotation settings as the ones used in the classification task. Our model uses the equivariant adaptation of DGCNN (see Section~\ref{sec:method:orient} for detail) as the backbone network. 
Hyper-parameters and other implementation details are put to supplementary materials.

\vspace{-1em}
\paragraph{Results}
Table \ref{tab:normal} shows the test cosine-distance errors ($1-\cos\langle \vn_\text{true}, \vn_\text{pred} \rangle$) \cite{zhang2020global} of all the methods to be compared on the normal estimation benchmark. 
It is obvious that the performance of other non-equivariant models degrades when we apply more flexible rotation on the point clouds.
Thanks to the equivariance, our model presents consistency of performance on different settings of train-test rotation, and outperforms both non-equivariant and equivariant baselines by a clear margin. This result reveals the inherent generalizability of our model with respect to arbitrary poses and confirms the benefit of using equivariant models to estimate equivariant properties.

% \vspace{-0.5em}
\subsection{N-Body System Modeling}
\label{sec:expr:nbody}
% \vspace{-0.5em}

\begin{table}
    \centering
    \caption{Mean squared errors (MSE) of predicted future coordinates in N-body system modeling task.} 
    \vspace{-0.8em}
    \label{tab:nbody}
    \begin{adjustbox}{width=0.7\linewidth}
    \begin{tabular}{l|cc}
\toprule
Methods & $10$ Particles & $20$ Particles  \\
\midrule
Linear(vel) \cite{satorras2021n} & 0.339 & 0.408 \\
GNN \cite{kipf2016semi} & 0.183 & 0.256 \\
\midrule
SE(3) Tr. \cite{fuchs2020se3} & 0.351 & 0.371 \\
TFN \cite{thomas2018tensor} & 0.236 & 0.273\\
EGNN \cite{satorras2021n} & 0.126 & 0.212 \\
\midrule
Ours & 0.103 & 0.206  \\

\bottomrule
\end{tabular}
    \end{adjustbox}
    % \vspace{-0.6em}
\end{table}

Modeling a dynamic particle system is a fundamental yet challenging task in many other research areas.
It is a also prototype task of many other specific tasks in various domains, such as molecule simulation.
Though controlled by simple physical rules, N-body dynamic systems exhibit complex dynamics and it is expensive to predict their future states solely based on simulation.
Therefore, the goal of this task is to predict each particle's future states (position and velocity) based on its current states. 

\vspace{-1em}
\paragraph{Setup}
We test our model and baselines on two systems: 10-particle system and 20-particle system.
For each system, we run the simulation program provided by \cite{fuchs2020se3} to collect particle trajectories. 
We collected 3,000 trajectories for training and 2,000 for test. Each trajectory has 1,000 timesteps.
The target is to predict the coordinates of all the particles after 1,000 timesteps, given their current positions, velocities, and types (positively charged or negatively charged).
We use the mean squared error (MSE) of the predicted future positions as the evaluation metric. 
Hyper-parameters and other implementation details of our model are put to supplementary materials.
Our baselines include a simple linear model introduced in \cite{satorras2021n}, a non-equivariant graph neural network model \cite{kipf2018neural}, as well as equivariant models including tensor field networks \cite{thomas2018tensor}, SE(3) transforms \cite{fuchs2020se3}, and EGNN \cite{satorras2021n}.

\vspace{-1em}
\paragraph{Results}

As shown in Table \ref{tab:normal}, our model exhibits competitive performance, suggesting its power in modeling physical interactions among points.

\begin{figure}
    \centering
    \includegraphics[width=1.0\linewidth]{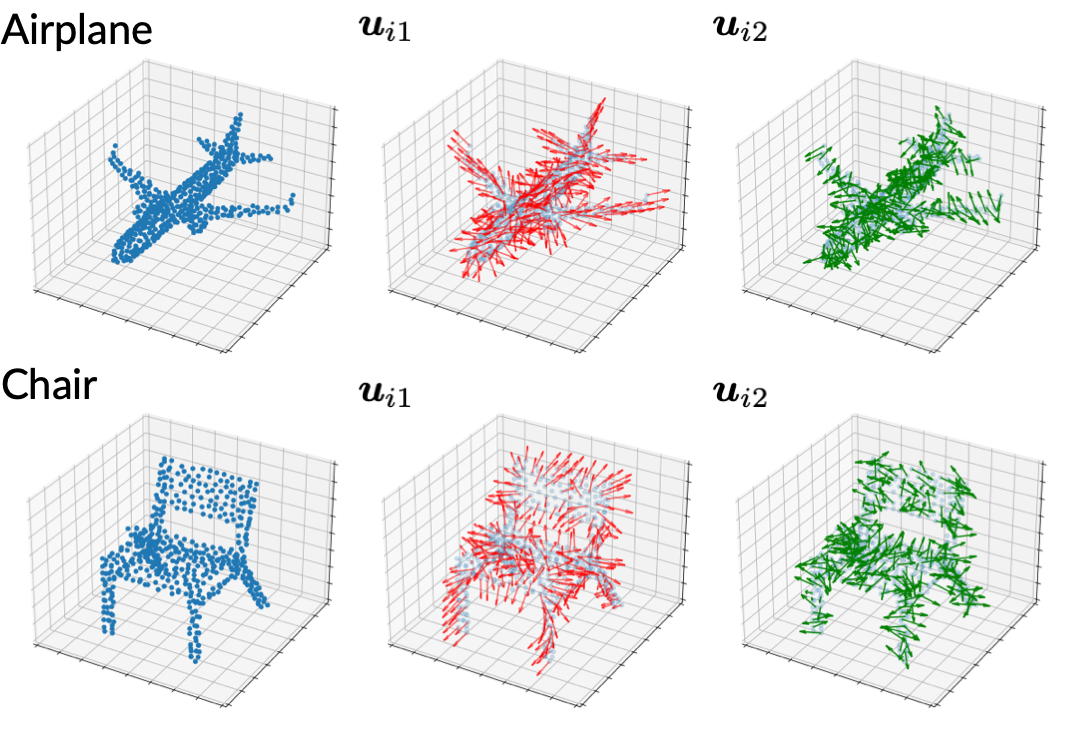}
    \caption{\textit{Left}: Input point clouds.~\textit{Middle \& Right}: The first and second column vectors of learned orientations in the classification model.}
    \label{fig:orientation}
    \vspace{-1.2em}
\end{figure}

\subsection{Visualizing Orientations}
\label{sec:expr:analytical}

We visualize two examples of learned orientations from the point cloud classification model in Figure~\ref{fig:orientation}.
We find that the learned orientations are correlated with the global structure. For example, the first orientation vectors (red arrows) on the airplane's wings and tails head to the direction to which the wings and tails stretch and the second orientation vectors (green arrows) are perpendicular to the wings.
In addition, the first orientation vectors (red) on the chair's both front legs consistently head to the front direction, indicating the learned orientations' awareness of structural similarities. 
\section{Conclusions and Limitations}
\label{sec:conclusion}
In this paper, we introduce a scheme for equivariant point set analysis.
The core ingredient is to learn point-wise orientations and project neighbor coordinates using the learned orientation when aggregating information.
By extensive experiments, we demonstrate our model's effectiveness and generality.

However, the proposed method requires an additional sub-network to estimate point-wise orientations. This is a non-negligible overhead and it is important to reduce the cost of orientation learning.
Another concern of the method is its robustness. It is unclear to what extent the learned orientations are vulnerable to noisy data, which is inevitable in real-world settings. It is necessary to study the robustness in future work.
Lastly, this paper only demonstrates few applications, it would be valuable to examine the method in other more challenging tasks such as molecular modeling.

\section*{Acknowledgement}
This work was supported by National Key R\&D Program of China No. 2021YFF1201600.

%%%%%%%%% REFERENCES
{\small
\bibliographystyle{ieee_fullname}
\bibliography{references}
}

% \iffalse

\clearpage
\newpage
\onecolumn
\appendix
\centerline{\Large\bf Supplementary Material}%
\vspace{0.8em}
\centerline{\large\bf Equivariant Point Cloud Analysis via Learning Orientations for Message Passing}%

\section{Orientation Learning Networks}

This section elaborates on the neural network architecture for learning orientations described in Section~\ref{sec:method:orient}.

The building block of the network is GVP-based graph convolutional networks (GVP-GCN) \cite{jing2020gvp}.
To understand GVP-GCN, we first briefly introduce GVP (geometric vector perceptron).
GVP is a generalization of the traditional perceptron (a linear layer followed by a non-linear function).
While traditional perceptrons take only an array of scalar features as input and outputs another array of scalar values, GVP takes an array of scalars and an array of 3D vectors as input, and output another array of scalars and array of vectors:
\begin{equation}
    (\vs^\prime, \mV^\prime) \gets \operatorname{GVP}(\vs, \mV), \quad \text{where } \vs = [s_1, \ldots, s_n] \in \sR^{n} \text{ and } \mV = [\vv_i, \ldots, \vv_m] \in \sR^{3\times m}.
\end{equation}
The most relevant property of $\text{GVP}$ is its equivariance to rotation:
\begin{equation}
\label{eq:gvp-eq}
    (\vs^\prime, \mR\mV^\prime) \gets \operatorname{GVP}(\vs, \mR\mV).
\end{equation}
The above formulation and equivariance property is sufficient for building the rest part of our model.
Therefore, if the reader is interested in more technical details and formal proofs about GVP, we refer the reader to \cite{jing2020gvp}.

Next, we introduce GVP-based graph convolutional networks (GVP-GCN).
Let $\vh_i$ denote the scalar features, $\mV_i$ denote the vector features of point $i$, and $\vh_{ij}$, $\mV_{ij}$ denote the scalar and vector of edge $ij$.
The message passing step of GVP-GCN is defined as:
\begin{align}
    (\vh_{(i \gets j)}, \mV_{(i \gets j)}) & := \operatorname{GVP}\left( \text{concat}(\vh_i, \vh_{ij}), \text{concat}(\mV_i, \mV_{ij}) \right), \\
    (\vh^\prime_{i}, \mV^\prime_{i}) & \gets (\vh_{i}, \mV_{i}) + \frac{1}{|\gN(i)|}\sum_{j \in \gN(i)} (\vh_{(i \gets j)}, \mV_{(i \gets j)}), \\
    (\vh^\text{out}_{i}, \mV^\text{out}_{i}) & \gets (\vh^\prime_{i}, \mV^\prime_{i}) + \operatorname{GVP}(\vh^\prime_{i}, \mV^\prime_{i}).
\end{align}
For short, we denote the above message passing scheme as $(\vh^\text{out}_{i}, \mV^\text{out}_{i}) \gets \operatorname{GVP-GConv}(\vh_i, \vh_{ij}, \mV_i, \mV_{ij})$.
As proven in \cite{jing2020gvp}, $\operatorname{GVP-GConv}$ is equivariant to rotation:
\begin{align}
    (\vh^\text{out}_{i}, \mR\mV^\text{out}_{i}) \gets \operatorname{GVP-GConv}(\vh_i, \vh_{ij}, \mR\mV_i, \mR\mV_{ij}).
\end{align}

To learn orientations from point clouds, we modify $\operatorname{GVP-GConv}$ according to the characteristics of 3D point clouds.
Specifically, we assign distance to the scalar feature of edges and coordinate difference to the vector feature of edges:
\begin{align}
    \vh_{ij} = [\| \vx_i - \vx_j\|], \ \mV_{ij} = [(\vx_i - \vx_j)].
\end{align}
The message passing formula over the point cloud (denoted as $\text{V-GConv}$) is then formulated as:
\begin{align}
    (\vh_{(i \gets j)}, \mV_{(i \gets j)}) & := \operatorname{GVP}\left( \text{concat}(\vh_i, [\| \vx_i - \vx_j\|]), \text{concat}(\mV_i, [(\vx_i - \vx_j)]) \right), \label{eq:vgconv-1} \\
    (\vh^\prime_{i}, \mV^\prime_{i}) & \gets (\vh_{i}, \mV_{i}) + \frac{1}{|\gN(i)|}\sum_{j \in \gN(i)} (\vh_{(i \gets j)}, \mV_{(i \gets j)}), \label{eq:vgconv-2} \\
    (\vh^\text{out}_{i}, \mV^\text{out}_{i}) & \gets (\vh^\prime_{i}, \mV^\prime_{i}) + \operatorname{GVP}(\vh^\prime_{i}, \mV^\prime_{i}). \label{eq:vgconv-3}
\end{align}
The formula involve point-wise coordinates $\mX$, scalar and vector features $\mH$, $\mV$, so we denote the layer as $(\vh^\text{out}_{i}, \mV^\text{out}_{i}) \gets \operatorname{V-GConv}(\mX, \vh_i, \mV_i)$ for short.
To ensure equivariance, we assign all zero as the initial scalar and vector features of the point cloud, as defined in Eq.\ref{eq:vgconv-layer1}.

\section{Proofs}

In this section, we prove the three propositions in Section~\ref{sec:method}.

\setcounter{propos}{0}
\begin{propos}
The proposed network for learning orientations, denoted as $f_\text{ort}(\vx_1, \ldots, \vx_N) = (\mO_1,\ldots,\mO_N)$, is equivariant to rotation and invariant to translation, satisfying $f_\text{ort}\left( \mR\vx_1 + \vt, \ldots, \mR\vx_N + \vt \right) = \left( \mR\mO_1, \ldots, \mR\mO_N \right)$. 
\end{propos}
\begin{proof}
First, we show the equivariance of the message passing formula defined by Eq.\ref{eq:vgconv-1}-\ref{eq:vgconv-3}.
Let $\mR \in \text{SO}(3)$ denote an arbitrary proper rotation matrix and $\vt \in \sR^3$ denote an arbitrary vector in the 3D space.

According to the equivariance of GVP (Eq.\ref{eq:gvp-eq}), Eq.\ref{eq:vgconv-1} is equivariant by:
\begin{align*}
    (\vh_{(i \gets j)}, \mR \mV_{(i \gets j)}) & = \operatorname{GVP}\left( \text{concat}(\vh_i, [\| \vx_i - \vx_j \|]), \text{concat}(\mR\mV_i, [ \mR(\vx_i - \vx_j) ]) \right) \\
    & = \operatorname{GVP}\left( \text{concat}(\vh_i, [\| \mR(\vx_i - \vx_j) \|]), \text{concat}(\mR\mV_i, [ \mR(\vx_i - \vx_j) ]) \right) \\
    & = \operatorname{GVP}\left( \text{concat}(\vh_i, [\| \mR\vx_i + \vt - \mR\vx_j - \vt \|]), \text{concat}(\mR\mV_i, [ \mR\vx_i + \vt - \mR\vx_j - \vt ]) \right) \\
    & = \operatorname{GVP}\left( \text{concat}(\vh_i, [\| (\mR\vx_i+\vt) - (\mR\vx_j+\vt)\|]), \text{concat}(\mR\mV_i, [ (\mR\vx_i+\vt) - (\mR\vx_j+\vt) ]) \right)
\end{align*}
Next, it is straightforward to see that Eq.\ref{eq:vgconv-2} is equivariant as it simply sums up vectors, and Eq.\ref{eq:vgconv-3} is equivariant by Eq.\ref{eq:gvp-eq}. By chaining the equivariance of Eq.\ref{eq:vgconv-1}-\ref{eq:vgconv-3}, we can see a $\operatorname{V-GConv}$ layer is equivariant. That is, formally:
\begin{align*}
    (\vh^\text{out}_{i}, \mR\mV^\text{out}_{i}) \gets \operatorname{V-GConv}(\mR\mX+\vt, \vh_i, \mR\mV_i).
\end{align*}
In addition, it is easy to see that stacking multiple $\text{V-GConv}$ layers results to an equivariant function:
\begin{align*}
    (\vh^\ts{L}_{i}, \mR\mV^\ts{L}_{i}) \gets \operatorname{V-GConv}_L(\mR\mX+\vt, \operatorname{V-GConv}_{L-1}(\mR\mX+\vt, ( \ldots \operatorname{V-GConv}_1(\mR\mX+\vt, \vh^\ts{0}_i, \mR\mV^\ts{0}_i) \ldots ))).
\end{align*}
Now, the problem is how to choose initial scalar and vector features $\vh^\ts{0}_i, \mV^\ts{0}_i$ for an input point cloud $\mX$.
From the above equation, we can see that the initial vector features $\mV^\ts{0}$ should be aligned to the rotation $\mR$ in order to preserve equivariance.
However, in most settings, the pose of the point cloud (parameterized by $\mR$ and $\vt$) is unknown.
Therefore, we simply set $\mV^\ts{0} = \vzero$ which is irrelevant to the global pose\footnote{There are some other choices for $\mV^\ts{0}$ such as local principle components of each point. These features align with the global rotation and can be used as initial vector features. In this work, we choose the simplest solution.}.
As for the initial scalar features $\mH^\ts{0} = [\vh^\ts{0}_1, \ldots, \vh^\ts{0}_N]$, they should be invariant to the global pose, so we also simply set them to $\vzero$. 
Setting both initial scalar and vector features to zero leads to the design of the first layer of the orientation learning network (Eq.\ref{eq:vgconv-layer1}).

So far, we have shown that our designed stack of $\text{V-GConv}$ layers is equivariant. 
Note that the final layer of the network outputs two vectors for each point. They are equivariant to rotation and invariant to translation according to the above results.

Next, we consider the equivariance of the Gram-Schmidt process for constructing orientation matrices (Eq.\ref{eq:ortho1}-\ref{eq:ortho3}).
For Eq.\ref{eq:ortho1}, we have:
\begin{align*}
    \widetilde{\vu}_{i1} := \frac{\mR\vv^\ts{L}_{i1}}{\| \mR\vv^\ts{L}_{i1} \|} = \frac{\mR\vv^\ts{L}_{i1}}{\| \vv^\ts{L}_{i1} \|} = \mR \vu_{i1},
\end{align*}
and for Eq.\ref{eq:ortho3}, we have:
\begin{align*}
    \widetilde{\vu}_{i2}  :=
    \frac{\widetilde{\vu}^\prime_{i2}}{\| \widetilde{\vu}^\prime_{i2} \|}
    = \mR\vv^\ts{L}_{i2} - \langle \mR\vv^\ts{L}_{i2}, \mR\vu_{i1}\rangle \mR\vu_{i1} 
    = \mR\left[ \vv^\ts{L}_{i2} - \langle \vv^\ts{L}_{i2}, \vu_{i1}\rangle \vu_{i1} \right]
    = \mR \vu_{i2}.
\end{align*}
These prove the equivariance of Eq.\ref{eq:ortho1}-\ref{eq:ortho3}.
For Eq.\ref{eq:ort_matrix}, we show its equivariance by:
\begin{equation*}
    \widetilde{\mO_i} := \left[ \mR \vu_{i1}, \mR \vu_{i2}, \mR\vu_{i1} \times \mR\vu_{i2} \right] =  \mR\left[ \vu_{i1}, \vu_{i2},\vu_{i1} \times \vu_{i2} \right] =  \mR \mO_i.
\end{equation*}

Finally, as the whole orientation learning network for is a composite of the $\text{V-GConv}$ network and the Gram-Schmidt-based matrix construction process.
Therefore, the network, denoted as $f_\text{ort}$, is equivariant to the rotation and invariant to translation, satisfying $f_\text{ort}\left( \mR\vx_1 + \vt, \ldots, \mR\vx_N + \vt \right) = \left( \mR\mO_1, \ldots, \mR\mO_N \right)$.
\end{proof}

\begin{propos}
Under the assumption that $(\mO_1, \ldots, \mO_N)$ comes from an equivariant function $g$ satisfying $g(\mR\mX + \vt) = (\mR\mO_1, \ldots, \mR\mO_N)$, and that $\mH$ comes from an invariant function $h$ satisfying $h(\mR\mX + \vt) = \mH$,
the message passing formula in Eq.\ref{eq:oriented_general}, denoted as $f^\ts{\ell}_\text{mp}$, is invariant to rotation and translation and satisfies $f^\ts{\ell}_\text{mp}\left(\mR\mX+\vt \right) = f^\ts{\ell}_\text{mp}\left(\mX \right)$.
\end{propos}
\begin{proof}
\begin{align*}
    \widetilde{\vh}_i^\ts{\ell+1} := f^\ts{\ell+1}_\text{mp}\left(\mR\mX+\vt \right) & = \underset{j \in \gN(i)}{\operatorname{Agg}} H\left( \vh^\ts{\ell}_i, \vh^\ts{\ell}_j, (\mR\mO)^\intercal_i\left((\mR\vx_j+\vt) - (\mR\vx_i+\vt)\right) \right) \\
    & = \underset{j \in \gN(i)}{\operatorname{Agg}} H\left( \vh^\ts{\ell}_i, \vh^\ts{\ell}_j, \mO^\intercal_i \cancel{\mR^\intercal \mR} (\vx_j - \vx_i) \right) \\
    & = \underset{j \in \gN(i)}{\operatorname{Agg}} H\left( \vh^\ts{\ell}_i, \vh^\ts{\ell}_j, \mO^\intercal_i (\vx_j - \vx_i) \right) \\
    & = f^\ts{\ell+1}_\text{mp}\left(\mX \right) = \vh_i^\ts{\ell+1}.
\end{align*}
Note that initial point-wise features $\vh_i^\ts{0}(i=1,\ldots,N)$ are all set to zero, which are invariant to $\mX$. Hence the above equation hold.
When $\ell > 0$, by induction, $\vh_i^\ts{\ell-1}(i=1,\ldots,N)$ are invariant to $\mX$, so the equations also hold.

\end{proof}

\begin{propos}
Under the assumption that $(\mO_1, \ldots, \mO_N)$ comes from an equivariant function $g$ satisfying $g(\mR\mX + \vt) = (\mR\mO_1, \ldots, \mR\mO_N)$, and that $\mH^\ts{L}$ comes from an invariant function $h$ satisfying $h(\mR\mX + \vt) = \mH^\ts{L}$,
The property estimator $f_\text{prop}$ defined by Eq.\ref{eq:equiv-property-1} and Eq.\ref{eq:equiv-property-2} is equivariant to rotation and invariant to translation, 
satisfying $f_\text{prop}\left(\mR\mX+\vt \right) = \mR f_\text{prop}\left(\mX \right)$.
\end{propos}
\begin{proof}
\begin{align*}
    \widetilde{\ve}_i := f_\text{prop}\left(\mR\mX+\vt \right) = \mR\mO_i\operatorname{MLP}(\vh_i^\ts{L}) = \mR f_\text{prop}\left(\mX \right) = \mR \ve_i.
\end{align*}
\end{proof}
\section{Additional Results}

\paragraph{Point Cloud Part Segmentation: SO(3)/SO(3)}
The following table is supplementary to Table \ref{tab:partseg} in Section \ref{sec:expr:pointcloud:partseg}, showing the part segmentation performance of our model and baselines in SO(3)-train/SO(3)-test setting. 
\vspace{-0.8em}
\begin{table*}[h]
    \centering
    \caption{Point cloud segmentation results in the IoU (\textit{Inter-over-Union}). In SO(3)/SO(3) setting, our model outperforms other baselines on 10 out of 16 shapes.} 
    \vspace{-0.8em}
    \label{tab:partseg_so3so3}
    \begin{adjustbox}{width=1.0\textwidth}
    \begin{tabular}{l|cccc cccc cccc cccc c}
\toprule

SO(3)/SO(3) & plane & bag & cap & car & chair & earph. & guitar & knife & lamp & laptop & motor & mug & pistol & rocket & skate & table  \\ 
\hline

PointNet~\cite{qi2017pointnet} & {81.6} & 68.7 &  74.0 &  70.3 &  87.6 &  68.5 &  88.9 &  80.0 &  74.9 &  83.6 &  56.5 &  77.6 &  75.2 &  53.9 &  69.4 &  79.9 \\

PointNet++ (MSG)~\cite{qi2017pointnet++}& 79.5 & 71.6 &  {\bf 87.7} &  70.7 &  {88.8} &  64.9 &  88.8 &  78.1 &  79.2 &  94.9 &  54.3 &  92.0 &  76.4 &  50.3 &  68.4 &  81.0 \\

PointCNN~\cite{li2018pointcnn} & 78.0 & 80.1 &  78.2 &  68.2 &  81.2 &  70.2 &  82.0 &  70.6 &  68.9 &  80.8 &  48.6 &  77.3 &  63.2 &  50.6 &  63.2 &  {\bf 82.0} \\

DGCNN~\cite{wang2019dynamic} & 77.7 & 71.8 &  77.7 &  55.2 &  87.3 &  68.7 &  88.7 &  {85.5} &  {81.8} &  81.3 &  36.2 &  86.0 &  {77.3} &  51.6 &  65.3 &  80.2 \\

ShellNet~\cite{zhang2019shellnet} & 79.0 & 79.6 &  80.2 &  64.1 &  87.4 &  71.3 &  88.8 &  81.9 &  79.1 &  {95.1} &  57.2 &  91.2 &  69.8 &  {55.8} &  73.0 &  79.3 \\ 
\hline

RI-Conv \cite{zhang2019rotation} & 80.6 & 80.2 &  70.7 &  68.8 &  86.8 &  70.4 &  87.2 &  84.3 &  78.0 &  80.1 &  57.3 &  91.2 &  71.3 &  52.1 &  66.6 &  78.5 \\

GC-Conv \cite{zhang2020global} & 
    81.2 & 82.6 & 81.6 & 70.2 & 88.6 & 70.6 & 86.2 & \bf 86.6 &
    81.6 & 79.6 & 58.9 & 90.8 & 76.8 & 53.2 & 67.2 & 81.6 \\

RI-Fwk. \cite{li2021rotation} & 81.4 & \bf 84.5 &  85.1 &  {75.0} &  88.2 &  {72.4} &  {90.7} &  84.4 &  80.3 &  \bf 84.0 &  \bf 68.8 &  {92.6} &  76.1 &  52.1 &  {74.1} &  80.0 \\

\hline

Ours & 
    \bf 81.8 & 78.8 & 85.4 & \bf 78.0 & \bf 89.6 & \bf 76.7 & \bf 91.6 & 85.7 &
    \bf 81.7 & 82.1 & 67.6 & \bf 95.0 & \bf 79.1 & \bf 63.5 & \bf 76.5 & 81.0 \\
\bottomrule

\end{tabular}
    \end{adjustbox}
    \vspace{-1.0em}
\end{table*}

\paragraph{Additional Visualization of Learned Orientations}
Figure \ref{fig:suppl-ort1} and \ref{fig:suppl-ort2} present more examples of learned orientations from the point cloud classification model. 
We observe that though not explicitly supervised, the learned orientations are correlated with the structure of shapes.
For example, from the first two rows of Figure \ref{fig:suppl-ort1}, we can see that red arrows on the airplane's wings and tails head to the direction to which the wings and tails stretch.
We also find that the orientations are aware of structural similarity --- the red arrows on chair legs consistently point to the front of the chair.
We also note that the learned orientations are not trivial. An example is the second shape on the bottom row of Figure \ref{fig:suppl-ort2}. The red arrows are obviously correlated to the vase's surface normals, rather than simply diverge from the center of the vase.

\begin{figure*}[h]
    \centering
    \includegraphics[width=0.8\textwidth]{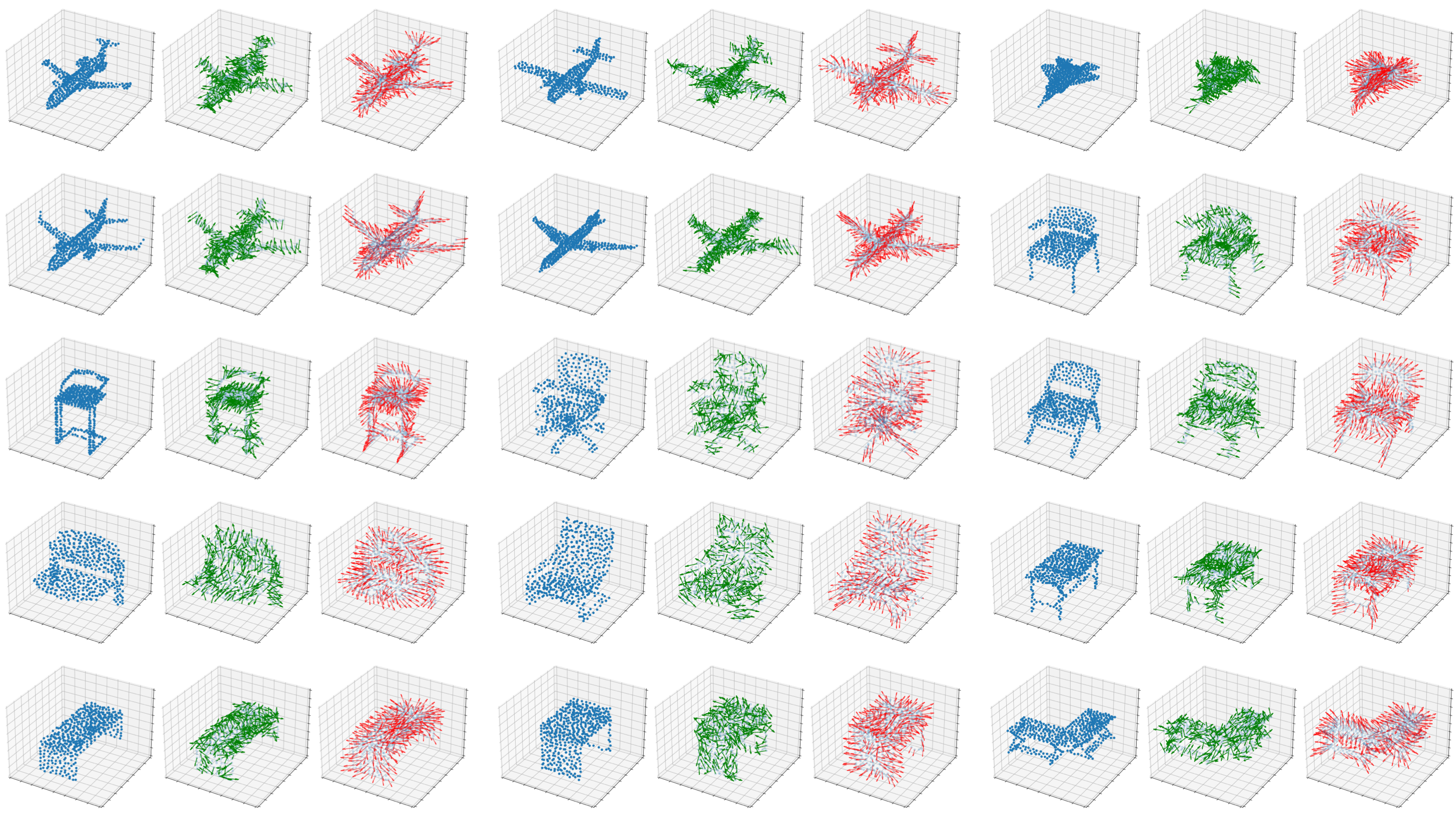}
    \caption{Green and red arrows indicate the first two column vectors of learned orientation matrices.}
    \label{fig:suppl-ort1}
    \vspace{-1.2em}
\end{figure*}

\begin{figure*}[h]
    \centering
    \includegraphics[width=0.8\textwidth]{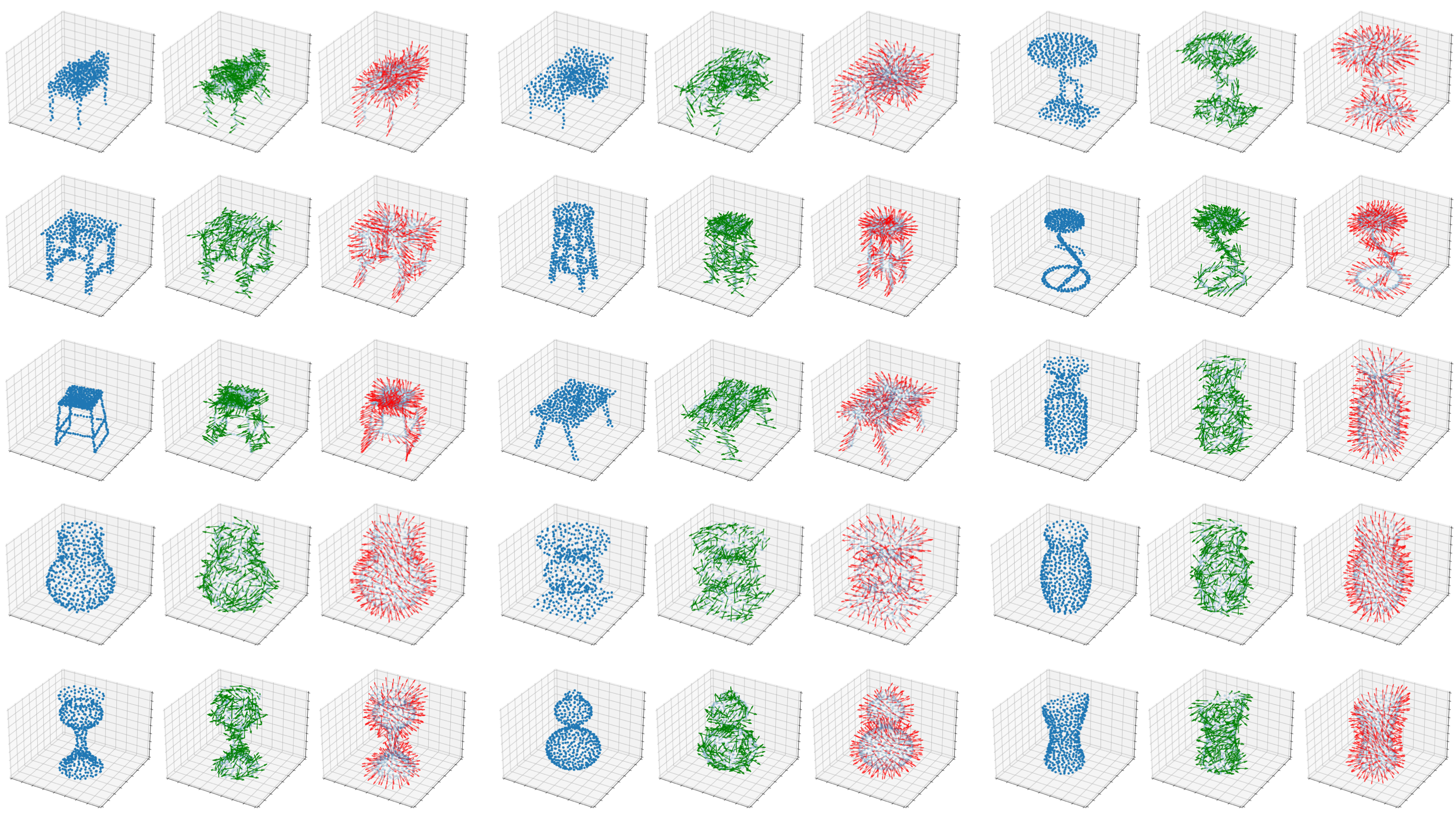}
    \caption{Green and red arrows indicate the first two column vectors of learned orientation matrices.}
    \label{fig:suppl-ort2}
    \vspace{-1.2em}
\end{figure*}

% \fi

\end{document}